\documentclass[a4paper,reqno,gray]{amsart}
\usepackage{cmap} 
\usepackage[british]{babel}
\usepackage[T1]{fontenc}
\usepackage{mathtools}
  \mathtoolsset{mathic}
  \allowdisplaybreaks
\usepackage{textcomp}
\usepackage[spacing]{microtype}
\usepackage{url}
\usepackage{tikz}
\usepackage{enumerate}
\usepackage{paralist}
\usepackage{amsmath,amssymb,amsfonts} 
\usepackage{amsthm}
\usepackage{mathtools}
\usepackage{mathptmx}
\usepackage{helvet}
\usepackage{courier}
\usepackage{microtype}
\usepackage{hyperref,url}
  \hypersetup{  
    bookmarksnumbered
  }
  \hypersetup{  
    breaklinks=false,
    pdfborderstyle={/S/U/W 0},  
    citebordercolor=.235 .702 .443,
    urlbordercolor=.255 .412 .882,
    linkbordercolor=.804 .149 .19,
  }
  \hypersetup{ 
    pdfauthor={Jasper De Bock and Gert de Cooman},
    pdftitle={Credal nets under epistemic irrelevance},
    pdfkeywords={credal net, epistemic irrelevance}
  }
\usepackage[all]{hypcap} 
\usepackage{fixltx2e} 
\usepackage{wrapfig}
\usepackage{ifthen}

\usepackage{nicefrac}
\usepackage{tikz}
\usetikzlibrary{calc}
\usetikzlibrary{decorations.pathmorphing}
\usetikzlibrary{decorations.markings}
\usetikzlibrary{fit}
\usetikzlibrary{backgrounds}
\usetikzlibrary{trees}
\usetikzlibrary{matrix}

\xdefinecolor{rood}{RGB}{241, 90, 34}
\xdefinecolor{geel}{RGB}{255, 197, 11}
\xdefinecolor{roze}{RGB}{236, 0, 140}
\xdefinecolor{groen}{RGB}{179, 211, 53}
\xdefinecolor{grijs}{RGB}{72, 119, 116}

\newcommand{\printMessage}[3]{
\ifthenelse{\equal{#3}{left}}{
  \draw[message] ($(#1)!.30!(#2) !.2!90:(#1)$) -- ($(#1)!.70!(#2) !-.2!90:(#2)$) node[midway] (m) {};
}
{
  \draw[message] ($(#1)!.30!(#2) !.2!-90:(#1)$) -- ($(#1)!.70!(#2) !-.2!-90:(#2)$) node[midway] (m) {};
  
}
}
\tikzstyle{nd}=[circle,draw=black!80,fill=geel!15,minimum size=.65cm,line width=.2pt]
\tikzstyle{observed}=[nd,fill=groen!100]
\tikzstyle{target}=[nd,fill=roze!100]
\tikzstyle{unobserved}=[nd,fill=geel!15]
\tikzstyle{terminal}=[nd,fill=geel!15]
\tikzstyle{ed}=[draw=geel!60!black,line width=.8pt, postaction={decorate}, decoration={markings,mark=at position 1.0 with {\arrow[draw=geel!60!black,line width=.8pt]{>}}}]
\tikzstyle{backbone}=[draw=rood,line width=1.8pt, postaction={decorate}, decoration={markings,mark=at position 1.0 with {\arrow[draw=rood,line width=.8pt]{>}}}]
\tikzstyle{message}=[draw=grijs,line width=.5pt, postaction={decorate}, decoration={markings,mark=at position 1.0 with {\arrow[draw=grijs,line width=.8pt]{>}}}]

\usepackage{wrapfig}
\usepackage{ifthen}
\usepackage{listings}
\usepackage{enumitem}
\usepackage{paralist}
\usepackage{mdwlist}

\newtheorem{theorem}{Theorem}
\newtheorem{proposition}[theorem]{Proposition}
\newtheorem{lemma}[theorem]{Lemma}

\newcommand*{\D}{\mathcal{D}}
\newcommand*{\M}{\mathcal{M}}
\DeclarePairedDelimiter{\set}{\{}{\}}
\newcommand*{\cset}[3][]{\set[#1]{#2:#3}}
\newcommand*{\nats}{\mathbb{N}}

\newcommand*{\natex}[2][]{\mathcal{E}_{#1}({#2})}
\newcommand*{\A}{\mathcal{A}}

\newcommand{\reals}{\mathbb{R}}

\newcommand{\nodes}{G}
\newcommand{\nonterminals}{\nodes^\lozenge}
\newcommand{\roots}{\nodes^\square}

\newcommand{\parents}[1]{P(#1)}
\newcommand{\mother}[1]{m(#1)}
\newcommand{\children}[2][]{C_{#1}(#2)}

\newcommand{\descendants}[1]{D(#1)}

\newcommand{\precedes}{\sqsubseteq}
\newcommand{\sprecedes}{\sqsubset}

\newcommand{\nonparnondes}[1]{N(#1)}

\newcommand{\var}[1]{X_{#1}}
\newcommand{\values}[1]{\mathcal{X}_{#1}}

\newcommand{\gamblesX}{\mathcal{G}(\mathcal{X})}
\newcommand{\gambles}[1]{\mathcal{G}(\values{#1})}

\newcommand{\gamblesposX}{\gamblesX_{>0}}
\newcommand{\gamblesnegX}{\gamblesX_{\leq 0}}

\newcommand{\xval}[1]{x_{#1}}
\newcommand{\zval}[1]{z_{#1}}

\newcommand{\sing}[1]{\{#1\}}
\newcommand{\xsing}[1]{\sing{\xval{#1}}}

\newcommand{\ind}[1]{\mathbb{I}_{#1}}

\DeclareMathOperator{\posi}{posi}
\DeclareMathOperator{\marg}{marg}

\title{Credal nets under epistemic irrelevance}
\author{Jasper De Bock}
\author{Gert de Cooman}
\email{\{jasper.debock,gert.decooman\}@UGent.be}
\address{Ghent University, SYSTeMS Research Group, Technologiepark 914,  9052 Zwijnaarde, Belgium.}
\begin{document}

\begin{abstract}
We present a new approach to credal nets, which are graphical models
that generalise Bayesian nets to imprecise probability. Instead of
applying the commonly used notion of strong independence, we replace
it by the weaker notion of epistemic irrelevance. We show how
assessments of epistemic irrelevance allow us to construct a global
model out of given local uncertainty models and mention some useful
properties. The main results and proofs are presented using the
language of sets of
desirable gambles, which provides a very general and expressive way of
representing imprecise probability models. 
\end{abstract}


\maketitle

\section{Introduction}
\label{sec:introduction}

This paper is under construction. At the current stage, it only aims to
present some essential ideas and theorems. We intend to extend this
preliminary work to a full size paper in the near future.

\section{Sets of desirable gambles}
\label{sec:SDGs}
Consider a variable $X$ taking values in some non-empty and finite set
$\values{}$. Knowledge about the possible values this variable may assume
can be modelled in various ways: probability mass functions, credal
sets and coherent lower previsions are only a few of the many
options. We choose te use a different approach, being a set of
desirable gambles. We will model a
subject's beliefs regarding the value of a variable $X$ by
means of his behaviour: which gambles
(or bets) on the unknown value of $X$ would our subject be
inclined to participate in?

Although they are not as well known as
other (imprecise) probability models, sets of desirable gambles have a series of advantages. To begin with, sets of
desirable gambles are more expressive then both credal sets and lower previsions. For example, sets
of desirable gambles are easily able to model such things as conditioning on
events with probability zero, which is something other imprecise
probability models cannot do. Secondly, sets of desirable gambles have
the advantage of being operational, meaning that there is
a practical way of constructing a model that represents the subject's beliefs. In the case of sets of desirable gambles this can be done by
offering the subject certain gambles and asking him whether or not he
wants to participate. And finally, it tends to be much easier to construct proofs in the
language of coherent sets of desirable gambles then it is to do so in
other languages. We will give a brief survey of the basics of sets of desirable gambles
and refer to Refs.~\cite{cooman2010,couso2011,walley2000} for more
details and further discussion.

\subsection{Desirable gambles.}
A gamble $f$ is a real-valued map on $\values{}$ which is interpreted as
an uncertain reward. If the value of the variable $X$ turns out to be $x$, the (possibly negative) reward is $f(x)$. 
A non-zero gamble is called desirable if we accept the transaction in which
\begin{inparaenum}[(i)]
  \item the actual value $x$ of the variable is determined, and
  \item we receive the reward $f(x)$.
\end{inparaenum}
The zero gamble is not considered to be desirable, mainly because we want desirability to represent a strict preference to the zero gamble.

We will model a subject's beliefs regarding the possible values
$\values{}$ that a variable $X$ can assume by means of a set $\D$ of desirable gambles, which will be a subset of the set $\gamblesX$ of all gambles on $\values{}$. 
For any two gambles $f$ and $g$ in $\gamblesX$, we say that $f\geq g$
if ${f(x)\geq g(x)}$ for all~$x$ in~$\values{}$ and $f>g$ if $f\geq g$
and $f\neq g$. We use~$\gamblesposX$ to denote the set of all gambles
$f\in\gamblesX$ for wich $f>0$ and $\gamblesnegX$ to denote the set of
all gambles $f\in\gamblesX$ for which $f\leq 0$. As a special kind of
gambles we consider the \emph{indicator functions}. For every event $A\subseteq\values{}$,
the gamble $\ind{A}$ is called the indicator function of
$A$. It is equal to $1$ if the event occurs (the variable $X$ assumes
a value in $A$) and zero otherwise. 
\subsection{Coherence.}

In order to represent a rational subject's beliefs about the
values a variable can assume, a set $\D\subseteq\gamblesX$ of desirable gambles should satisfy some rationality requirements. If these requirements are met, we call the set $\D$ \emph{coherent}. \\[2ex]
For all $f$, $f_1$, $f_2\in\gamblesX$ and all real
  $\lambda>0$:
  \begin{enumerate}[label=\upshape D\arabic*.,ref=\upshape D\arabic*,leftmargin=*]
  \item\label{item:Dnonzero} if $f\leq 0$ then $f\notin\D$
  \item\label{item:Dapg} if $f>0$ then $f\in\D$
  \item\label{item:Dscl} if $f\in\D$ then $\lambda f\in\D$\hfill [scaling]
  \item\label{item:Dcmb} if $f_1,f_2\in\D$ then $f_1+f_2\in\D$\hfill [combination]
  \end{enumerate}
Requirements~\ref{item:Dscl} and~\ref{item:Dcmb} make~$\D$ a convex cone: $\posi(\D)=\D$, where we have used the positive hull operator $\posi$ which generates the set of finite strictly positive linear combinations of elements of its argument set:
\begin{equation}\label{eq:posi}
  \posi(\D)\coloneqq\cset[\bigg]{\sum_{k=1}^n\lambda_kf_k}{f_k\in\D,\lambda_k\in\reals^+_0,n\in\nats_0}.
\end{equation}
Here $\reals^+_0$ is the set of all strictly positive real numbers, and
$\nats_0$ the set of all natural numbers (positive integers).

\subsection{Natural extension.}

In practice, a set of desirable gambles will usually be elicited by
presenting an expert a number of gambles and asking him
wether or not he finds them desirable. This results in a (finite) assessment
of desirable gambles $\A\subseteq\gamblesX$ and the question raises
whether this can be extended to a coherent set.
It is shown in Ref.~\cite{walley1991} that if the assessment $\A$ can be extended to a coherent set of desirable gambles,
the smallest (most conservative) such coherent set is given by $\natex{\A}{}\coloneqq \posi(\A\cup\gamblesposX)$
and we then call $\natex{\A}{}$ the natural extension of $\A$.

\subsection{Maximal sets of desirable gambles.}\label{sec:maximal}

A coherent set $\D$ of desirable gambles on $\values{}$ is called \emph{maximal} if
it is not strictly included in any other coherent set of desirable gambles on~$\values{}$. In other words, if adding any gamble $f$ to $\D$ makes sure
we can no longer extend the set $\D\cup\{f\}$ to a set that is still
coherent. We will denote maximal sets of desirable gambles as $\M$ instead
of using the general notation $\D$. 

These maximal sets of desirable gambles have a number of useful
properties. 
For example, a coherent set $\D$ of desirable gambles on $\values{}$ is allways the
intersection of all the maximal coherent sets $\M$ of desirable
gambles on $\values{}$ that include $\D$; see Ref.~\cite{walley1991}. In
other words, $f\in\D$ if and only if $f\in\M$ for every $\M\supseteq\D$. As a consequence, if a gamble
$f\in\gambles{}$ is not an element of $\D$, there is at least one
maximal set $\M\supseteq\D$ for which $f\notin\M$.
Another useful property that holds for every maximal
set $\M$ is that for all gambles $f\neq 0$ in $\gamblesX$,
either $f$ or $-f$ is an element of $\M$; see
Ref.~\cite{couso2011}.


\section{Credal nets under epistemic irrelevance}
\label{sec:forircredalnets}

\subsection{Directed acyclic graphs.}
A directed acyclic graph (DAG) is a graphical model that is well known
for its use in bayesian networks. It consists of a finite set of
nodes (vertices), which are joint together into a network by a set of
directed edges, each edge connecting one node with another. Since this
directed graph is assumed to be acyclic, it is not possible to follow a sequence of edges from node to node and end up back
at the same node you started from.

We will call $\nodes$ the set of nodes $s$ associated with a given
DAG. For two nodes $s$ and $t$, if there is a directed edge from $s$
to $t$, we say that $s$ is a \emph{parent} of $t$ and $t$ is a \emph{child} of
$s$. Note that a single node can have multiple parents and multiple children.
For any node $s$, its set of parents is denoted by
$\parents{s}$ and its set of children by $\children{s}$.
If a node 
$s$ 
has no parents, we use the convention
$\mother{s}=\emptyset$ and we call it a \emph{root node}. The set of
all root nodes is
denoted as $\roots\coloneqq\cset{s\in\nodes}{\parents{s}=\emptyset}$.
If $\children{s}=\emptyset$, then we call $s$ a \emph{leaf}, or \emph{terminal node}. 
We denote by $\nonterminals\coloneqq\cset{s\in\nodes}{\children{s}\neq\emptyset}$ the set of all non-terminal nodes.
\par
For nodes $s$ and $t$, we write $s\precedes t$ if \emph{$s$ precedes
  $t$}, i.e., if there is a directed segment (sequence of directed edges) in the graph from $s$ to
$t$. If $s\precedes t$ and $s\neq t$, we say that $s\sprecedes t$.
For any node $s$, we denote its set of \emph{descendants} by
$\descendants{s}\coloneqq\cset{t\in\nodes}{s\sprecedes t}$
and its set of \emph{non-parent
non-descendants} is given by $\nonparnondes{s}\coloneqq\nodes\setminus(\parents{s}\cup\{s\}\cup\descendants{s})$.

\par
\subsection{Variables and gambles on them.}
With each node $s$ of the tree, there is associated a variable $\var{s}$ assuming values in a non-empty finite set $\values{s}$. 
We denote  by $\gambles{s}$ the set of all gambles on $\values{s}$. 
We extend this notation to more complicated situations as follows. 
If $S$ is any subset of $\nodes$, then we denote by $\var{S}$ the tuple of variables whose components are the $\var{s}$ for all $s\in S$. 
This new joint variable assumes values in the finite set
$\values{S}\coloneqq\times_{s\in S}\values{s}$ and the corresponding
set of gambles is denoted by $\gambles{S}$. When $S=\emptyset$, we let
$\values{\emptyset}$ be a singleton. The
corresponding variable $\var{\emptyset}$ can then only assume this
single value, so there is no uncertainty about it. $\gambles{\emptyset}$ can then be identified with the set $\reals$ of real numbers.
Generic elements of $\values{s}$ are denoted by $\xval{s}$ or
$\zval{s}$ and similarly for $\xval{S}$ and $\zval{S}$ in $\values{S}$. 
Also, if we mention a tuple $\zval{S}$, then for any $t\in S$, the corresponding element in the tuple will be denoted by $\zval{t}$. 
We assume all variables in the network to be logically independent, meaning that the variable $\var{S}$ may assume \emph{all} values in $\values{S}$, for all $\emptyset\subseteq S\subseteq\nodes$. 
\par
We will frequently use the simplifying device of identifying a gamble $f_S$ on $\values{S}$ with its \emph{cylindrical extension} to $\values{U}$, where $S\subseteq U\subseteq\nodes$. 
This is the gamble $f_U$ on $\values{U}$ defined by $f_U(\xval{U})\coloneqq f_S(\xval{S})$ for all $\xval{U}\in\values{U}$. 
To give an example, if $\mathcal{K}\subseteq\gambles{\nodes}$, this trick allows us to consider $\mathcal{K}\cap\gambles{S}$ as the set of those gambles in $\mathcal{K}$ that depend only on the variable $\var{S}$. 
As another example, this device allows us to identify the gambles $\ind{\xsing{S}}$ and $\ind{\xsing{S}\times\values{\nodes\setminus S}}$, and therefore also the events $\xsing{S}$ and $\xsing{S}\times\values{\nodes\setminus S}$. 
More generally, for any event $A\subseteq\values{S}$, we can identify the gambles $\ind{A}$ and $\ind{A\times\values{\nodes\setminus S}}$, and therefore also the events $A$ and $A\times\values{\nodes\setminus S}$.
\par
\subsection{Modelling our beliefs about the network.}
\label{sec:beliefmodelling}
Throughout the paper, we consider sets of desirable
gambles as models for a subject's beliefs about the values that
certain variables in the network may assume. One of the main
contributions of this paper, further on in Section~\ref{sec:joint}, will be to show
how to construct a joint model for our network, being a coherent set
$\D_{\nodes}$ of desirable gambles on $\values{\nodes}$.

From such a joint model, one can derive both conditional and marginal models.
Let us start by explaining
how to condition the global model $\D_\nodes$. Consider a subset $I$ of $\nodes$
and assume we want to update the model $\D_{\nodes}$ with the
information that $\var{I}=\xval{I}$. This leads to the following updated set of desirable
gambles:
\begin{equation*}
\D_{\nodes}\rfloor\xval{I}\coloneqq\cset[\big]{f\in\gambles{\nodes\setminus
    I}}{\ind{\xsing{I}}f\in\D_{\nodes}},
\end{equation*}
which represents our subject's beliefs
about the value of the variable $\var{\nodes\setminus I}$, conditional on the
observation that $\var{I}$ assumes the value $\xval{I}$.
This definition is very intuitive, since $\ind{\xsing{I}}f$
is the unique gamble that is called off (is equal to zero) if $\var{I}\neq\xval{I}$ and equal to
$f$ if $\var{I}=\xval{I}$.
Notice that since $\ind{\xsing{\emptyset}}=1$, the special case of
conditioning on the certain variable $\var{\emptyset}$ does not yield
any problems. As
wanted, it amounts to not conditioning at all. 

Marginalisation is also very intuitive in the language of sets of
desirable gambles. Suppose we want to derive a marginal model for our
subject's beliefs about the variable $\var{O}$, where $O$ is some
subset of $\nodes$. This can be done by using the set of desirable
gambles that belong to $\D_{\nodes}$ but only depend on the variable
$\var{O}$:
\begin{equation*}
\marg_{O}(\D_\nodes)\coloneqq\cset[\big]{f\in\gambles{O}}{f\in\D_\nodes}.
\end{equation*}

Now let $I$ and $O$ be \emph{disjoint}  subsets of $\nodes$ and let
$\xval{I}$ be any element of $\values{I}$. By sequentially applying
the process of conditioning and marginalisation we can obtain
conditional marginal models for our subject's beliefs about the value of the variable $\var{O}$, conditional on the observation that $\var{I}$ assumes the value $\xval{I}$:
\begin{equation}\label{eq:condmarg}
\marg_{O}(\D_{\nodes}\rfloor\xval{I})=\cset[\big]{f\in\gambles{O}}{\ind{\xsing{I}}f\in\D_\nodes}.
\end{equation}

Since coherence is trivially preserved under both conditioning and
marginalisation, we find that if the joint model $\D_\nodes$ is
coherent, all the derived models will also
be coherent. 

Conditional and/or marginal models do not necessarily have to be
derived from a joint model, they can instead also be given as seperate
models on their own. In that case we will generecally denote them as
$\D_{O\rfloor\xval{I}}$. The special case of an unconditional marginal
model is sometimes denoted as $\D_O$ but we will also use the general notation above by letting
$I=\emptyset$ in the general notation above.

\par

\subsection{Epistemic irrelevance.}
\label{sec:irrelevance}
We now have the necessary tools to introduce one of the most important
concepts for this paper, that of epistemic irrelevance. We describe the case of conditional irrelevance, as we will show that the unconditional
version of epistemic irrelevance can easily be recovered as a special
case.
\par
Consider three disjoint subsets $C$, $I$, and $O$ of $\nodes$. 
When a subject judges $\var{I}$ to be \emph{epistemically irrelevant to $\var{O}$ conditional on $\var{C}$}, he assumes that if he knows the value of $\var{C}$, then learning in addition which value $\var{I}$ assumes in $\values{I}$ will not affect his beliefs about $\var{O}$. 
More formally, assume that a subject has for every
$\xval{C}\in\values{C}$ a coherent conditional set of desirable
gambles $\D_{O\rfloor\xval{C}}$ on $\values{O}$.
If he assesses $\var{I}$ to be epistemically irrelevant to $\var{O}$
conditional on $\var{C}$, this implies that he can infer from these
models $\D_{O\rfloor\xval{C}}$ the following additional conditional models
$\D_{O\rfloor\xval{C\cup I}}$ on $\values{O}$:
\begin{equation*}
  \D_{O\rfloor\xval{C\cup I}}=\D_{O\rfloor\xval{C}}
  \text{ for all $\xval{C\cup I}\in\values{C\cup I}$.}
\end{equation*}
By now, it should be clear that it suffices for the unconditional case, in the discussion above, to let $C=\emptyset$. This makes sure the variable $\var{C}$ has only one possible value, so conditioning on that variable amounts to not conditioning at all.

\subsection{Local uncertainty models.}
\label{sec:local}
We now add \emph{local uncertainty models} to each of the nodes $s$ in
our network. These local models are assumed to be given beforehand and
will be used further on in Section~\ref{sec:joint} as basic building blocks
to construct a joint model for a given network.

If $s$ is not a root node of the network, i.e.~has a non-empty set of parents
$\parents{s}$, then we have a conditional local model for every
instantiation of its parents. For each $\xval{\parents{s}}\in\values{\parents{s}}$, we have a conditional coherent set $\D_{s\rfloor\xval{\parents{s}}}$ of
desirable gambles on $\values{s}$. It represents our subject's beliefs
about the variable $\var{s}$ conditional on the information
that its parents $\var{\parents{s}}$ assume the value $\xval{\parents{s}}$.

If $s$ is one of the root nodes, i.e.~has no parents, then our subject's local
beliefs about the variable $\var{s}$ are represented by an
unconditional local model. It should be a coherent set of desirable
gambles and will be denoted by $\D_{s}$.
As was explained in Section~\ref{sec:beliefmodelling}, we can also use the common generic notation $\D_{s\rfloor\xval{\parents{s}}}$ in
this unconditional case, since for a root node $s$, its set of parents
$\parents{s}$ is
equal to the empy set $\emptyset$.

\subsection{The interpretation of the graphical model.}\label{sec:graphical:interpretation}
In classical Bayesian nets, the graphical structure is taken to
represent the following assessments: for any node $s$, conditional on
its parent variables, the associated variable 
is independent
of its non-parent non-descendant variables
.

When generalising this interpretation to imprecise graphical
networks, the classical notion of
independence gets replaced by a more general, imprecise notion of
independence that is usually chosen to be strong independence. In this
paper we will not do so, we choose to use the weaker, assymetric notion of epistemic irrelevance
instead, which was introduced earlier on in
Section~\ref{sec:irrelevance}. In the special case of precise
uncertainty models, both epistemic irrelevance and strong independence
will reduce to the usual classical notion of independence and the
corresponding interpretations of the graphical network are equivalent
with the one used in a classical Bayesian network.

\par
In the present context, we assume that the graphical structure of the
network embodies the following conditional irrelevance assessments,
turning the network into a \emph{credal net under epistemic irrelevance}.
Consider any node $s$ in the network, its set of parents $\parents{s}$
and its set of non-parent non-descendants $\nonparnondes{s}$. 
Then \emph{conditional on its parent variables $\var{\parents{s}}$,
  the non-parent non-descendant variables $\var{\nonparnondes{s}}$ are
  assumed to be epistemically irrelevant to the variable $\var{s}$
  associated with the node $s$.} 

For a coherent set of desirable gambles $\D_\nodes$ that describes our
subject's global beliefs about all the variables in the network, this interpretation has the following consequences.
It can easily be seen from Sections~\ref{sec:beliefmodelling}
and~\ref{sec:irrelevance} that it implies for all $s\in\nodes$ and all
subsets $I$ of $\nonparnondes{s}$ that
\begin{equation}\label{eq:irrelevance}
  \marg_{s}(\D_{\nodes}\rfloor\xval{\parents{s}\cup I})=\marg_{s}(\D_{\nodes}\rfloor\xval{\parents{s}})
  \text{ for all $\xval{\parents{s}\cup I}\in\values{\parents{s}\cup I}$.}
\end{equation}

\par

\section{Constructing the most conservative joint}
\label{sec:joint}
Let us now show how to construct a global model for the variables in
the network, and argue that it is the most conservative coherent model
that extends the local models and expresses all conditional
irrelevancies encoded in the network.
But before we do so, let us provide some motivation. Suppose we have a
global set of desirable gambles $\D_\nodes$, how do we express that
such a model is compatible with the assessments encoded in the network? 
\subsection{Defining properties of the joint.}
We will require our joint model to satisfy the following four properties.
First of all, we require that our global model extends the local
ones. This means that the local models derived from the
global one should be equal to the given local models:
\begin{enumerate}[label=\upshape G1.,ref=\upshape G1]
\item\label{item:global:local}
  For each node $s$ in $\nodes$,
  $\marg_{s}(\D_{\nodes}\rfloor\xval{\parents{s}})=\D_{s\rfloor\xval{\parents{s}}}$
  for all $\xval{\parents{s}}\in\values{\parents{s}}$.
\end{enumerate}
The second requirement is that our model reflects all epistemic
irrelevancies encoded in the graphical structure of the network:
\begin{enumerate}[label=\upshape G2.,ref=\upshape G2]
\item\label{item:global:irrelevance} $\D_\nodes$ satisfies all
  equalities that are imposed by Eq.~\eqref{eq:irrelevance}. In these
  equalities, the right hand side can be replaced by
  $\D_{s\rfloor\xval{\parents{s}}}$ due to requirement \ref{item:global:local}.
\end{enumerate}
The third requirement is that our model satisfies the rationality requirement of coherence:
\begin{enumerate}[label=\upshape G3.,ref=\upshape G3]
\addtocounter{enumi}{1} 
\item\label{item:global:coherence}
$\D_\nodes$ is coherent (satisfies requirements \ref{item:Dnonzero}--\ref{item:Dcmb}).
\end{enumerate}
Since requirements
\ref{item:global:local}--\ref{item:global:coherence} do not uniquely
determine a global model, there is also a final requirement, which guarantees that all inferences we make on the basis of our global models are as conservative as possible, and are therefore based on no other considerations than what is encoded in the tree:
\begin{enumerate}[label=\upshape G4.,ref=\upshape G4]
\addtocounter{enumi}{3} 
\item\label{item:global:smallest}
$\D_\nodes$ is the smallest set of desirable gambles on
$\values{\nodes}$ satisfying requirements
\ref{item:global:local}--\ref{item:global:coherence}: it is a subset
of any other set that satisfies them.
\end{enumerate}
We will now show how to construct the unique global model $\D_\nodes$ that satisfies
all off the four requirements
\ref{item:global:local}--\ref{item:global:smallest} that were given above.

\subsection{Constructing the joint.}

Let us start by looking at a single given marginal model $\D_{s\rfloor\xval{\parents{s}}}$ and
investigate some of its implications for the joint model $\D_\nodes$. Consider
any node~$s$ in the network and fix values
$\xval{\parents{s}}$ for its parents. For the local model
$\D_{s\rfloor\xval{\parents{s}}}$, we now introduce a corresponding (non-coherent)
set $\A^{irr}_{s\rfloor\xval{\parents{s}}}$ of desirable gambles on
$\values{\nodes}$:
\begin{equation}\label{eq:localAirr}
\A^{irr}_{s\rfloor\xval{\parents{s}}}\coloneqq\cset[\big]{\ind{\xsing{\parents{s}\cup\nonparnondes{s}}}f}{\xval{\nonparnondes{s}}\in\values{\nonparnondes{s}},~f\in\D_{s\rfloor\xval{\parents{s}}}}.
\end{equation}
It shall become clear by the proposition below that such a set
$\A^{irr}_{s\rfloor\xval{\parents{s}}}$ indeed contains (some of the)
implications that follow from the local model
$\D_{s\rfloor\xval{\parents{s}}}$. Next, we bundle all these local
implications:
\begin{align}\label{eq:globalAirr}
\A^{irr}_\nodes
\coloneqq&\hspace{-10pt}\bigcup_{{s\in\nodes,\,\xval{\parents{s}}\in\values{\parents{s}}}}
\hspace{-10pt}\A^{irr}_{s\rfloor\xval{\parents{s}}}.
\end{align}
This results in a set $\A^{irr}_\nodes$ of desirable gambles on
$\values{\nodes}$ that will become essential further on for our construction of
the joint model $\D_\nodes$. The importance of this set $\A^{irr}_\nodes$
is already manifested by the following proposition, which is proven in
Appendix A (as are all other important results of this paper).
\begin{proposition}\label{prop:localirr}
Consider any $s\in\nodes$ and
$\xval{\parents{s}}\in\values{\parents{s}}$. Then the set
$\A^{irr}_{s\rfloor\xval{\parents{s}}}$ will be a subset of any joint model
$\D_\nodes$ satisfying requirements \ref{item:global:local} and
\ref{item:global:irrelevance}. As a consequence, their union
$\A^{irr}_\nodes$ will also be a subset of any joint model $\D_\nodes$
satisfying requirements \ref{item:global:local} and
\ref{item:global:irrelevance} and thus a subset of the unique joint
model that satisfies all four requirements \ref{item:global:local}--\ref{item:global:smallest}.
\end{proposition}
\noindent
We now propose the following expression for the joint model
$\D_\nodes$, describing our subject's beliefs about the variables in
the network and satisfying all four requirements \ref{item:global:local}--\ref{item:global:smallest}:
\begin{equation}\label{eq:joint}
\D_\nodes\coloneqq\posi(\A^{irr}_\nodes).
\end{equation}
\noindent
Since our eventual joint model $\D_\nodes$ should be coherent
(satisfy requirement \ref{item:global:coherence}), and thus in particular
should be a convex cone (satisfy properties \ref{item:Dscl} and \ref{item:Dcmb}), we know that
$\posi(\D_\nodes)$ should be equal to $\D_\nodes$. It is therefor
very intuitive to consider the set given above, since
$\posi(\posi(\D))=\posi(\D)$ for any set of desirable gambles $\D$.
On the other hand, it is not obvious that this set is indeed the
unique joint model $\D_\nodes$ satisfying all four requirements
\mbox{\ref{item:global:local}--\ref{item:global:smallest}}. Therefore, the next part of
this paper consists of three propositions that will lead to the
main theorem, which states that the joint model $\D_\nodes$ does satisfy all four requirements
\ref{item:global:local}--\ref{item:global:smallest}. We start by
showing that it contains all positive gambles.
\begin{proposition}\label{prop:global:positive}
$\gambles{\nodes}_{>0}$ is a subset of $\posi(\A^{irr}_\nodes)$. As a
consequence, we have that
\begin{equation*}
\posi(\A^{irr}_\nodes)=\posi\left(\A^{irr}_\nodes\cup\gambles{\nodes}_{>0}\right)=:\natex{\A^{irr}_\nodes}
\end{equation*}
\end{proposition}
\noindent
This proposition serves as a first step towards the following
coherence result, which proves that our joint model satisfies
requirement \ref{item:global:coherence}.

\begin{proposition}\label{prop:global:coherence}
$\posi(\A^{irr}_\nodes)$ is a coherent set of desirable gambles on $\values{\nodes}$.
\end{proposition}
\noindent
The proof is given in Appendix A, but it contains an interesting
result that deserves to be pointed out. The crucial step of the proof
hinges on the assumption that if the local models of our network were
precise probability mass functions, we would be able to construct a
joint probability mass function that satisfies all irrelevancies (in that case
independecies) that are encoded in our network. Since the precise
version of a credal tree under epistemic irrelevance is a classical
Bayesian network, this assumption is indeed true. However, what is
nice about this approach, is that it can easily be extended to credal
networks with irrelevance assumptions that differ from the ones we
use, as long as the assumption above is satisfied. This enables us to
use existing coherence results for precise networks to proof their
counterparts for credal networks.

We now turn to an important proposition that will be essential to
prove that our joint model extends the local models and expresses all conditional
irrelevancies encoded in the network (satisfies requirements
\ref{item:global:local} and \ref{item:global:irrelevance}).
\begin{proposition}\label{prop:global:irrelevance}
Consider any $s\in\nodes$ and any subset $I$ of its non-parent non-descendants $\nonparnondes{s}$. If
we fix a value $\xval{\parents{s}\cup I}\in\values{\parents{s}\cup I}$,
then it holds for every $f\in\gambles{s}$ that
\begin{equation*}
\ind{\xsing{\parents{s}\cup I}}f\in\posi(\A^{irr}_\nodes)\Leftrightarrow f\in\D_{s\rfloor\xval{\parents{s}}}.
\end{equation*}
\end{proposition}
\noindent
We now have all necessary tools to formulate our most important result. It
is the main contribution of this paper and provides a justification
for the joint model $\D_\nodes$ that was proposed by Eq.~\eqref{eq:joint}.

\begin{theorem}\label{theo:global-model}
Consider any credal network under epistemic irrelevance with given
conditional marginal models $\D_{s\rfloor\xval{\parents{s}}}$, then $\D_\nodes=\posi(\A^{irr}_\nodes)$ is the unique set of desirable gambles on
$\values{\nodes}$ that satisfies all four requirements \ref{item:global:local}--\ref{item:global:smallest}.
\end{theorem}

\section{Conclusions}
\label{sec:conclusion}

This paper has presented a new approach to credal nets. We replaced the
commonly used notion of strong independence with the weaker notion of
epistemic irrelevance and expressed both our local models and the
eventual joint model in the language of sets of desirable
gambles. This has lead to an intuitive, easy expression for a joint
model, that is proven to be the most conservative coherent model
that extends the local models and expresses all conditional
irrelevancies encoded in the network.





\appendix
\section{Proofs of important results}\label{sec:app}
In this Appendix, we 
give proofs for
Propositions~\ref{prop:localirr},~\ref{prop:global:positive},~\ref{prop:global:coherence}
and~\ref{prop:global:irrelevance} and
Theorem~\ref{theo:global-model}.

\begin{proof}[Proof of Proposition~\ref{prop:localirr}]
The second part of this proposition is trivial and we thus only need to prove the first part. To do so, consider any $s\in\nodes$ and
$\xval{\parents{s}\cup\nonparnondes{s}}\in\values{\parents{s}\cup\nonparnondes{s}}$. As a consequence of requirements~\ref{item:global:local} and
\ref{item:global:irrelevance}, we see that $\marg_{s}(\D_{\nodes}\rfloor\xval{\parents{s}\cup\nonparnondes{s}})$ should be equal
to the given local model $\D_{s\rfloor\xval{\parents{s}}}$. If we now
apply Eq.~\eqref{eq:condmarg}, it follows immediately that
$\ind{\xsing{\parents{s}\cup\nonparnondes{s}}}f$ is an element of
$\D_\nodes$, thereby completing the proof.
\end{proof}

\begin{proof}[Proof of Proposition~\ref{prop:global:positive}]
The essential step is to see that for any
$\xval{\nodes}\in\values{\nodes}$, the indicator function
$\ind{\xsing{\nodes}}$ is an element of $\A^{irr}_\nodes$. To prove
this, pick an arbitrary leaf $s\in\nodes$. This is possible
because a DAG with a finite amount of nodes always has at least one leaf. Since $s$ is a leaf, it has no
descendants and we therefore have that
$\nodes=s\cup\parents{s}\cup\nonparnondes{s}$. Due to the coherence of
the local models, and in particular property \ref{item:Dapg}, the
indicator function $\ind{\xsing{s}}$ is an element of
$\D_{s\rfloor\xval{\parents{s}}}$. We can now apply
Eqs.~\eqref{eq:localAirr} and \eqref{eq:globalAirr} to see that
$\ind{\xsing{\nodes}}=\ind{\xsing{s\,\cup\parents{s}\cup\nonparnondes{s}}}$
is an element of $\A^{irr}_{\nodes}$. 

Since every $f>0$ is a
finite strictly positive linear combination of the indicator
functions that were constructed above, it follows that $\posi(\A^{irr}_\nodes)$ does indeed contain all positive
gambles in $\gambles{\nodes}_{>0}$.
As a consequence, we have that
$\posi(\A^{irr}_\nodes)=\posi(\A^{irr}_\nodes)\cup\gambles{\nodes}_{>0}$
and because $\posi(\posi(\D))=\posi(\D)$ for any set of desirable
gambles $\D$, we find that
$\posi(\A^{irr}_\nodes)=\posi\left(\posi(\A^{irr}_\nodes)\cup\gambles{\nodes}_{>0}\right)$. The
right hand side of this equality is trivially equal to
$\posi\left(\A^{irr}_\nodes\cup\gambles{\nodes}_{>0}\right)=:\natex{\A^{irr}_\nodes}$,
thereby completing the proof.
\end{proof}

Our proof of Proposition~\ref{prop:global:coherence} uses the
following convenient version of the separating hyperplane theorem. It
is proven in Ref.~\cite[Lemma 2]{cooman2011} and repeated here to make
the paper more self-contained.

\begin{lemma}\label{lemma:shpt}
Consider any finite subset $\A$ of $\gambles{}$. Then
$0\notin\natex{\A}\coloneqq\posi(\A\cup\gamblesposX)$ if and only if
there is some probability mass function $p$ such that
$\sum_{x\in\values{}} p(x)f(x)>0$ for all $f\in\A$ and $p(x)>0$ for all $x\in\values{}$.
\end{lemma}

\begin{proof}[Proof of Proposition~\ref{prop:global:coherence}]
Proving that $\posi(\A^{irr}_\nodes)$ is coherent, means showing that
it satisfies the properties \ref{item:Dnonzero}--\ref{item:Dcmb}. Property \ref{item:Dapg} is a direct consequence
of Proposition~\ref{prop:global:positive} and the
properties \ref{item:Dscl} and \ref{item:Dcmb} are trivial since
$\posi(\A^{irr}_\nodes)$ is a convex cone due to the use of the
$\posi$ operator. We thus only need to prove the first property,
stating that any gamble $f\in\gambles{\nodes}$ for which $f\leq 0$ can
not be an element of $\posi(\A^{irr}_\nodes)$. 

So consider any
$f\in\posi(\A^{irr}_\nodes)$ and assume \emph{ex absurdo} that
$f\leq 0$. We will show that this leads to a contradiction.
Since $f$ is an element of $\posi(\A^{irr}_\nodes)$, it follows from
Eqs.~\eqref{eq:posi},~\eqref{eq:localAirr} and~\eqref{eq:globalAirr} that 
\begin{equation}\label{eq:f}
f=\sum_{s\in\nodes\;}
\hspace{-2pt}
\sum_{~~\xval{\parents{s}}\in\values{\parents{s}}}
\hspace{-2pt}
\sum_{~~\xval{\nonparnondes{s}}\in\values{\nonparnondes{s}}}
\hspace{-6pt}
\ind{\xsing{\parents{s}\cup\nonparnondes{s}}}
f_{s,\,\xval{\parents{s}},\,\xval{\nonparnondes{s}}},
\end{equation}
where every $f_{s,\,\xval{\parents{s}},\,\xval{\nonparnondes{s}}}$ is
an element of $\D_{s\rfloor\xval{\parents{s}}}\cup\{0\}$ and at least
one of them differs from zero. The only perhaps surprising fact about
the equation above, is that it does not contain any (strictly positive) scaling factors
$\lambda_{\,s,\,\xval{\parents{s}},\,\xval{\nonparnondes{s}}}$. The
reason why these factors can be omitted is that the local models $\D_{s\rfloor\xval{\parents{s}}}$ are coherent and thus invariant under
strictly positive linear scaling. Therefore, scaling a gamble
$f_{s,\,\xval{\parents{s}},\,\xval{\nonparnondes{s}}}\in\D_{s\rfloor\xval{\parents{s}}}$ with a strictly
positive factor
$\lambda_{\,s,\,\xval{\parents{s}},\,\xval{\nonparnondes{s}}}$ will
still yield a gamble in $\D_{s\rfloor\xval{\parents{s}}}$.

Next, for every $s\in\nodes$ and
$\xval{\parents{s}}\in\values{\parents{s}}$ we construct a finite
subset of the local model $\D_{s\rfloor\xval{\parents{s}}}$:
\begin{equation*}
\A^{f}_{s\rfloor\xval{\parents{s}}}
\coloneqq
\cset[\big]{f_{s,\,\xval{\parents{s}},\,\xval{\nonparnondes{s}}}}{\xval{\nonparnondes{s}}\in\values{\nonparnondes{s}}\text{
    and } f_{s,\,\xval{\parents{s}},\,\xval{\nonparnondes{s}}}\neq 0}.
\end{equation*}
Due to the coherence of $\D_{s\rfloor\xval{\parents{s}}}$, we have
that
$0\notin\natex{\A^{f}_{s\rfloor\xval{\parents{s}}}}\subseteq\natex{\D_{s\rfloor\xval{\parents{s}}}}=\D_{s\rfloor\xval{\parents{s}}}$
and we can therefore apply Lemma~\ref{lemma:shpt}. This gives us for
every $s\in\nodes$ and
$\xval{\parents{s}}\in\values{\parents{s}}$ a mass function
$p_s(\cdot\vert\xval{\parents{s}})$ on $\values{s}$ with expectation
operator $E_s(\cdot\vert\xval{\parents{s}})$ on $\gambles{s}$ such that
$p_s(x_s\vert\xval{\parents{s}})>0$ for all $x_s\in\values{s}$ and
$E_s(g\vert\xval{\parents{s}})>0$ for each
$g\in\A^{f}_{s\rfloor\xval{\parents{s}}}$.

The trick is now to create a Bayesian network that has the conditional
mass functions $p_s(\cdot\vert\xval{\parents{s}})$ as its local models
and has the same graphical
structure as our credal net under epistemic irrelevance. If we let $E_\nodes$ be the expectation
operator for this Bayesian net, we find that
\begin{align*}
E_{\nodes}(f)&=
\sum_{s\in\nodes\;}
\hspace{-2pt}
\sum_{~~\xval{\parents{s}}\in\values{\parents{s}}}
\hspace{-2pt}
\sum_{~~\xval{\nonparnondes{s}}\in\values{\nonparnondes{s}}}
\hspace{-6pt}
E_\nodes\big(\ind{\xsing{\parents{s}\cup\nonparnondes{s}}}
f_{s,\,\xval{\parents{s}},\,\xval{\nonparnondes{s}}}\big)\\
&=
\sum_{s\in\nodes\;}
\hspace{-2pt}
\sum_{~~\xval{\parents{s}}\in\values{\parents{s}}}
\hspace{-2pt}
\sum_{~~\xval{\nonparnondes{s}}\in\values{\nonparnondes{s}}}
\hspace{-6pt}
E_\nodes\big(\ind{\xsing{\parents{s}\cup\nonparnondes{s}}}\big)
E_\nodes\big(f_{s,\,\xval{\parents{s}},\,\xval{\nonparnondes{s}}}\big\vert\;\xval{\parents{s}}\big)\\
&=
\sum_{s\in\nodes\;}
\hspace{-2pt}
\sum_{~~\xval{\parents{s}}\in\values{\parents{s}}}
\hspace{-2pt}
\sum_{~~\xval{\nonparnondes{s}}\in\values{\nonparnondes{s}}}
\hspace{-6pt}
p_\nodes(\xval{\parents{s}\cup\nonparnondes{s}})
E_\nodes\big(f_{s,\,\xval{\parents{s}},\,\xval{\nonparnondes{s}}}\big\vert\;\xval{\parents{s}}\big)
,
\end{align*}
in which $p_\nodes$ is the global mass function of the Bayesian
net. Since all the local probabilities
$p_s(\cdot\vert\xval{\parents{s}})$ are strictly positive, this is
also true for the global ones and we find that
$p_\nodes(\xval{\parents{s}\cup\nonparnondes{s}})>0$. For the
conditional expectations
$E_\nodes(f_{s,\,\xval{\parents{s}},\,\xval{\nonparnondes{s}}}\vert\;\xval{\parents{s}})$
there are two possibilities. Either
$f_{s,\,\xval{\parents{s}},\,\xval{\nonparnondes{s}}}=0$, in which
case
$E_\nodes(f_{s,\,\xval{\parents{s}},\,\xval{\nonparnondes{s}}}\vert\;\xval{\parents{s}})=0$,
either
$f_{s,\,\xval{\parents{s}},\,\xval{\nonparnondes{s}}}\in\A^{f}_{s\rfloor\xval{\parents{s}}}$,
in which case
$E_\nodes(f_{s,\,\xval{\parents{s}},\,\xval{\nonparnondes{s}}}\vert\;\xval{\parents{s}})>0$. However,
since at least one of
the gambles $f_{s,\,\xval{\parents{s}},\,\xval{\nonparnondes{s}}}$ in
Eq.~\eqref{eq:f} has to differ from zero, it is not possible that
$E_\nodes(f_{s,\,\xval{\parents{s}},\,\xval{\nonparnondes{s}}}\vert\;\xval{\parents{s}})=0$
for all gambles
$f_{s,\,\xval{\parents{s}},\,\xval{\nonparnondes{s}}}$ and we
can conclude that $E_\nodes(f)>0$. If we now apply our assumption ex
absurdo that $f\leq 0$ and thus $E_\nodes(f)\leq 0$, this leads to a
contradiction and completes the proof.
\end{proof}

\begin{proof}[Proof of Proposition~\ref{prop:global:irrelevance}]
The reverse implication is trivial due to the way $\posi(\A^{irr}_\nodes)$ is
constructed; see Eqs.~\eqref{eq:posi},~\eqref{eq:localAirr} and~\eqref{eq:globalAirr}. It therefore suffices to
prove the direct implication.
Consider any $s\in\nodes$, any subset $I$ of its non-parent
non-descendants $\nonparnondes{s}$ and fix a value
$\xval{\parents{s}\cup I}\in\values{\parents{s}\cup I}$. We set out to
proof for every $f\in\gambles{s}$ that
$f\notin\D_{s\rfloor\xval{\parents{s}}}$ implies
$\ind{\xsing{\parents{s}\cup I}}f\notin\posi(\A^{irr}_\nodes)$.

The case $f=0$ is trivial because $\ind{\xsing{\parents{s}\cup I}}f$
is then equal to zero, which can not be an element of
$\posi(\A^{irr}_\nodes)$ 
due to its
coherence; see
Proposition~\ref{prop:global:coherence}. If $f\neq 0$, we start by
applying some of the properties of maximal coherent sets of desirable
gambles that were introduced in Section~\ref{sec:maximal}. Due to the
first property, we can infer from
$f\notin\D_{s\rfloor\xval{\parents{s}}}$ that there is at least one
maximal set of desirable gambles $\M^*_{s\rfloor\xval{\parents{s}}}\supseteq\D_{s\rfloor\xval{\parents{s}}}$
for which $f\notin\M^*_{s\rfloor\xval{\parents{s}}}$. Due to the second
property and the fact that $f\neq 0$, this in turn implies that $-f\in\M^*_{s\rfloor\xval{\parents{s}}}$.

We now denote by $\A^{irr*}_\nodes$ the set that is obtained by
Eq.~\eqref{eq:globalAirr} if we replace the local model
$\D_{s\rfloor\xval{\parents{s}}}$ by the specific maximal superset
$\M^*_{s\rfloor\xval{\parents{s}}}$ that was introduced above. It
should be clear that $\A^{irr*}_\nodes\supseteq\A^{irr}_\nodes$.
Next, since $-f\in\M^*_{s\rfloor\xval{\parents{s}}}$, it follows from the
construction of $\A^{irr*}_\nodes$ that $\ind{\xsing{\parents{s}\cup
    I}}(-f)\in\A^{irr*}_\nodes\subseteq\posi(\A^{irr*}_\nodes)$. The
proof can now be completed if we realise that $\ind{\xsing{\parents{s}\cup
    I}}f\notin\posi(\A^{irr*}_\nodes)$ because this would contradict
with its coherence and notice that it implies that
$\ind{\xsing{\parents{s}\cup I}}f\notin\posi(\A^{irr}_\nodes)$ because
$\A^{irr*}_\nodes\supseteq\A^{irr}_\nodes$.
\end{proof}

\begin{proof}[Proof of Theorem~\ref{theo:global-model}]

We start by proving that the joint model $\D_\nodes=\posi(\A^{irr}_\nodes)$ satisfies
requirements~\ref{item:global:local}
and~\ref{item:global:irrelevance}. To do so, consider any
$s\in\nodes$, $I\subseteq\nonparnondes{s}$ and $\xval{\parents{s}\cup
  I}\in\values{\parents{s}\cup I}$ and an arbitrary gamble \mbox{$h\in\gambles{s}$}. It can be seen from the following
chain of equivalences that
$\marg_{s}(\D_{\nodes}\rfloor\xval{\parents{s}\cup
  I})=\D_{s\rfloor\xval{\parents{s}}}$.
\begin{align*}
h\in\marg_{s}(\D_{\nodes}\rfloor\xval{\parents{s}\cup
  I})
\Leftrightarrow
h\in\marg_{s}(\posi(\A^{irr}_\nodes)\rfloor\xval{\parents{s}\cup
  I})
&\Leftrightarrow
\ind{\xsing{\parents{s}\cup I}}h\in\posi(\A^{irr}_\nodes)\\
&\Leftrightarrow
h\in\D_{s\rfloor\xval{\parents{s}}}.
\end{align*}
The second equivalence is a direct application of
Eq.~\ref{eq:condmarg} and the third one is due to
Proposition~\ref{prop:global:irrelevance}.
Requirement~\ref{item:global:local} is now proven by letting
$I=\emptyset$ and requirement~\ref{item:global:irrelevance} is fulfilled
because $\marg_{s}(\D_{\nodes}\rfloor\xval{\parents{s}\cup
  I})=\D_{s\rfloor\xval{\parents{s}}}=\marg_{s}(\D_{\nodes}\rfloor\xval{\parents{s}})$. The
next step is to show that the joint model
$\D_\nodes=\posi(\A^{irr}_\nodes)$ also satisfies requirements~\ref{item:global:coherence} and \ref{item:global:smallest}.

Requirement \ref{item:global:coherence} demands that $\D_\nodes=\posi(\A^{irr}_\nodes)$ is
coherent, but since this is proven in
Proposition~\ref{prop:global:coherence}, the only thing that is left
to prove is requirement \ref{item:global:smallest}. This final requirement demands that
$\D_\nodes=\posi(\A^{irr}_\nodes)$ is included in any set of
desirable gambles satisfying the
requirements~\ref{item:global:local}--\ref{item:global:coherence}.
This is easy to proof since we know from Proposition~\ref{prop:localirr} that $\A^{irr}_\nodes$ is
a subset of any joint model satisfying all four requirements. It then
follows from the coherence requirement \ref{item:global:coherence}
that $\D_\nodes=\posi(\A^{irr}_\nodes)$ is a subset of all joint
models satisfying~\ref{item:global:local}--\ref{item:global:coherence}
and thus the unique smallest model that also satisfies
requirement~\ref{item:global:smallest}.
\end{proof}

\end{document}